\documentclass{article}

\usepackage[utf8]{inputenc}
\usepackage{microtype}
\usepackage{booktabs}
\usepackage{amsmath}
\usepackage{amsfonts}
\usepackage{amsthm}
\newtheorem{proposition}{Proposition}
\usepackage{xcolor}
\usepackage{url}
\usepackage{hyperref}
\hypersetup{hidelinks}

\usepackage[accepted]{icml2026}
\makeatletter
\renewcommand{\ICML@appearing}{\textit{Accepted to DEMO 2026: ICML Workshop on Decision-Making from Offline Datasets to Online Adaptation. Non-archival report.}}
\makeatother

\icmltitlerunning{ARCA: Adapter-Residual Credit Assignment}

\begin{document}

\twocolumn[
  \icmltitle{ARCA: Adapter-Residual Credit Assignment\texorpdfstring{\\}{ }When Token Signals Degenerate}

  \begin{icmlauthorlist}
    \icmlauthor{Rodney Lafuente-Mercado}{scale}
  \end{icmlauthorlist}

  \icmlaffiliation{scale}{Scale AI}
  \icmlcorrespondingauthor{Rodney Lafuente-Mercado}{rodney.lafuente@scale.com}

  \icmlkeywords{Machine Learning, ICML}

  \vskip 0.3in
]

\printAffiliationsAndNotice{}

\begin{abstract}\noindent Token-level credit assignment for language-model
reinforcement learning is usually formulated as if the policy were fully
trainable, while practical LLM-RL pipelines often rely on
parameter-efficient fine-tuning, especially LoRA. We argue that this
separation hides a structural failure mode. Under LoRA, the policy is
restricted to a low-rank neighborhood of the reference model, so the
per-token output-distribution differences used by common intrinsic credit
signals, surprisal, entropy reduction, and policy divergence, can become
degenerate after within-trajectory normalization, either approaching uniform
weights or concentrating on a small set of task-agnostic positions. We
formalize this behavior and propose measuring it directly with concentration
diagnostics such as weight Gini and effective-token ratio. We then introduce \emph{Adapter-Residual Credit
Assignment} (ARCA), a lightweight alternative that derives token salience from
the adapter's own hidden-state residual,
$\|h^{\text{adapted}}_t - h^{\text{base}}_t\|_2$. ARCA asks where the adapter
actually changes the model, rather than where the output distribution appears
uncertain or shifted, and requires no learned reward model, value head, or tree
construction. In a compact MATH/Qwen3-1.7B GRPO sweep, ARCA exhibits the
predicted non-degenerate middle-regime credit distribution under matched
rollout budgets and remains competitive with rank-matched baselines.
\end{abstract}

\paragraph{Code.}
Code is available at
\url{https://github.com/rodlaf/tokenweighting}.

\section{Introduction}

Reinforcement learning has become a central component of large language model
(LLM) post-training, particularly for alignment and for reasoning-oriented
tasks with verifiable rewards. A persistent challenge across these settings is
\emph{credit assignment}: trajectories are long, rewards are sparse and
outcome-level, and it is not obvious how a single scalar signal at the end of
a generation should be distributed across hundreds or thousands of token
decisions. Recent work has responded with a rapidly growing toolbox of
token-level credit-assignment methods, including entropy-aware modulation,
process reward models, tree-based prefix values, and reward
redistribution~\citep{cui2025prime,li2024red,tran2025tempo,wang2025beyond8020,he2026eapo,yu2026erpo,kazemnejad2024vineppo}.

In parallel, many open-source and academic LLM-RL pipelines use
parameter-efficient fine-tuning, especially LoRA~\citep{hu2021lora}, because of
memory and compute constraints; widely used frameworks such as TRL and verl
support this workflow. Recent results show that LoRA + RL can be dramatically
more sample- and parameter-efficient than full fine-tuning~\citep{wang2025tina}. Yet
the two research threads are developed in near-total isolation. Papers on
token-level credit assignment specify methods without reference to the
adaptation strategy, and PEFT is treated as an orthogonal implementation
detail.

In this work we argue that this separation is a mistake, and that the
\emph{interaction} between PEFT and credit assignment is itself a first-class
methodological issue. The core issue is geometric. Intrinsic token-level
weighting schemes derive per-token salience from quantities such as
surprisal, entropy reduction, or divergence between the policy and a reference
model. Under LoRA, however, the policy is constrained to a small low-rank
neighborhood of the reference, and the per-token differences that these
signals measure can lose the variation that made them useful. We formalize
this as degeneration of the normalized salience distribution, characterize it
with Gini coefficient and effective token count, and
show that the same mechanism applies across multiple
output-distribution-based weighting rules. This gives a mechanistic
explanation for the empirically observed failure of full-token
GRPO~+~LoRA~\citep{lee2025tokenefficientrl}: the weighting signal itself is
degraded before training ever begins.

Rather than trying to recover per-token structure from signals that LoRA
intrinsically flattens, we measure salience directly from the adapter's own
contribution to the forward pass. Adapter-Residual Credit Assignment (ARCA)
sets the salience at position $t$ equal to the norm of the adapter residual,
$\|h^{\text{adapted}}_t - h^{\text{base}}_t\|_2$, computed via a forward pass
with the adapter disabled. This signal is positive whenever the adapter is
active, non-uniform across positions because adapter input activations are
non-uniform, and requires no additional networks, learned reward models, or tree
construction.

Concretely, this paper makes three contributions:
\begin{enumerate}
    \item We identify and formalize a PEFT-specific failure mode of token
    credit assignment: under LoRA, output-distribution salience can degenerate
    into uniform broadcast or spurious sparsity, even when the underlying RL
    objective is unchanged.
    \item We introduce ARCA, a lightweight adapter-residual credit assignment
    rule whose normalized token weights remain non-degenerate whenever the
    adapter has position-varying hidden-state impact.
    \item We validate the mechanism with concentration diagnostics and a
    matched MATH/Qwen3-1.7B sweep, separating downstream performance from the
    more basic question of whether a proposed token-credit signal survives the
    adaptation regime actually used in LLM-RL.
\end{enumerate}

The remainder of the paper develops these contributions. Section
\ref{sec:related-work} positions the work relative to PEFT for language RL
and token-level credit assignment. Section \ref{sec:methods} presents the
weighting schemes, explains why output-distribution signals degenerate under
LoRA, and defines ARCA. Section \ref{sec:experiments} reports the diagnostic
and performance comparison on MATH with Qwen3-1.7B under the seven-run
reported sweep. Extended related work and theoretical interpretation
appear in Appendices~\ref{app:extended-related-work}
and~\ref{sec:theoretical}.

\section{Related Work}\label{sec:related-work}

\subsection{PEFT and Low-Rank Adaptation in Language RL}\label{sec:rw-peft}

Parameter-efficient fine-tuning, and in particular LoRA~\cite{hu2021lora}, is a
dominant adaptation strategy in open-source LLM-RL work. Tina
demonstrated that tiny LoRA adapters on a 1.5B base model suffice to reach
DeepSeek-R1-class reasoning behavior when combined with GRPO, at a tiny fraction of
full fine-tuning cost~\cite{wang2025tina}. A broader systematic evaluation of PEFT
methods under RLVR~\cite{yin2025peftrlvr} benchmarks over a dozen PEFT variants
(including DoRA, AdaLoRA, MiSS, PiSSA, MiLoRA, VeRA and Rank-1) on DeepSeek-R1-Distill
models and finds that standard LoRA is not optimal; structural variants such as DoRA,
AdaLoRA and MiSS consistently outperform it, and SVD-initialized variants (PiSSA,
MiLoRA) suffer from \emph{spectral collapse}, a finding that is broadly compatible with
our own signal-degeneration analysis. Token-Efficient RL introduces critic-free LoRA-compatible
variants of GRPO (S-GRPO and T-SPMO) that focus training on a subset of tokens,
reporting large gains on small models where full-token GRPO~+~LoRA trains
unstably~\cite{lee2025tokenefficientrl}; we now provide a mechanistic explanation for
that instability via signal degeneration. \emph{LoRA as an Implicit KL Regularizer}
analyzes how LoRA restricts the policy to a rank-constrained neighborhood of the
reference, deriving an explicit rank-dependent upper bound on the KL divergence between
policy and reference throughout training~\cite{anon2026implicitklra}. Our Section
\ref{sec:collapse} builds directly on this observation: an implicit KL bound is the
same mechanism that forces per-token log-probability differentials to be small,
which is the formal starting point for our degeneration result.

On the analysis side, \emph{Narrow Fine-Tuning Traces} shows that narrow
fine-tuning leaves clearly readable traces in the activation differences between
base and fine-tuned hidden states, and that the fine-tuning domain can be
recovered from those differences alone using simple diffing tools such as
patchscopes and activation steering~\cite{minder2025narrowfinetuning}. This is
directly relevant to ARCA because the adapter residual
$h^{\mathrm{adapted}}_t - h^{\mathrm{base}}_t$ is a per-position activation
difference of exactly the kind they study; their results supply empirical
evidence that this difference carries meaningful, semantically nontrivial
content, the property our method exploits. Concurrent work on
TopLoRA studies how to concentrate LoRA capacity on a small number of
high-impact tokens via token-wise input-output projections~\cite{li2025toplora}. To our knowledge, no prior work has characterized the
interaction between LoRA and token-level credit assignment, nor proposed an
adaptation-aware intrinsic weighting scheme.

\subsection{Positioning of This Work}

The literature makes three points clear. First, critics are not strictly required for
effective LLM post-training: critic-free estimators such as RLOO, ReMax, GRPO, and GSPO
are already competitive in both RLHF and RLVR
\cite{ahmadian2024backtobasics,li2023remax,shao2024deepseekmath,zheng2025gspo}. Second,
many researchers have concluded that trajectory-level rewards are too coarse and have
introduced denser supervision through process reward models, redistribution rules, tree
structures, optimal baselines, temporal traces, or entropy-based modulation
\cite{li2024red,cui2025prime,tran2025tempo,cao2025scar,parthasarathi2025grpolambda,li2026otb,hu2026pspo,he2026eapo,yu2026erpo,meng2026sparsebutcritical}.
Third, the overwhelming majority of practical LLM-RL pipelines are \emph{trained with
LoRA}, and a focused sub-literature has studied PEFT-specific effects in this
regime~\cite{hu2021lora,wang2025tina,yin2025peftrlvr,anon2026implicitklra,lee2025tokenefficientrl,minder2025narrowfinetuning,li2025toplora}.

Our contribution is to bridge the second and third of these threads. Existing intrinsic
credit-assignment signals (surprisal, entropy, divergence) are not novel as standalone
objects, and we do not claim otherwise; they are our baselines. Our novel claim is that
these signals \emph{interact pathologically with the adaptation strategy that the field
often uses}: under LoRA they can degenerate into uniform broadcast or spurious sparsity,
so any paper that reports a null result for ``more elaborate intrinsic weighting versus
uniform'' under LoRA may be observing a LoRA artifact, not evidence against token-level credit
assignment. We formalize this, supply a unified diagnostic (Gini/EffN) that makes the
degeneration visible at training
time, and propose ARCA as an adaptation-aware alternative whose construction directly
avoids the failure mode we identify.

\section{Methods}\label{sec:methods}

We consider on-policy reinforcement learning for autoregressive language models with
trajectory-level rewards. Let $\pi_\theta$ denote a language model parameterized by
$\theta$, generating a completion $y=(y_1,\dots,y_T)$ conditioned on a prompt $x$. After
sampling a full trajectory, the model receives a scalar reward $R(x,y)\in\mathbb{R}$
computed by an external verifier, such as exact-answer correctness or unit-test pass
rate. Our objective is
\begin{equation}
J(\theta)=\mathbb{E}_{x \sim \mathcal{D},\, y \sim \pi_\theta(\cdot \mid x)}[R(x,y)],
\end{equation}
where $\mathcal{D}$ denotes the prompt distribution. We focus on the regime most
relevant to recent reasoning-model training: sparse outcome rewards, on-policy sampling,
and no learned value function.

\subsection{Policy Gradient with Trajectory-Level Reward}

For a fixed prompt $x$, the standard REINFORCE estimator is
\begin{equation}
g_{\mathrm{rf}}(x,y)=R(x,y)\sum_{t=1}^{T}\nabla_\theta \log \pi_\theta(y_t \mid y_{<t},x).
\end{equation}
In practice, a baseline $b(x)$ is subtracted to reduce variance without changing the
expected gradient as long as $b(x)$ does not depend on the sampled action at token $t$:
\begin{equation}
g_{\mathrm{base}}(x,y)=\left(R(x,y)-b(x)\right)\sum_{t=1}^{T}\nabla_\theta \log \pi_\theta(y_t \mid y_{<t},x).
\end{equation}

We use prompt-level multi-sample baselines. Given $K$ sampled completions
$\{y^{(i)}\}_{i=1}^{K}$ for the same prompt $x$ with rewards $\{R^{(i)}\}_{i=1}^{K}$:
\begin{equation}
b_{\mathrm{RLOO}}^{(i)}(x)=\frac{1}{K-1}\sum_{j \neq i}R^{(j)}
\end{equation}
is the leave-one-out baseline used in RLOO-style estimators, while GRPO-style
normalization can be written as
\begin{equation}
A_{\mathrm{GRPO}}^{(i)}(x)=\frac{R^{(i)}-\mu_R(x)}{\sigma_R(x)+\varepsilon},
\end{equation}
where $\mu_R(x)$ and $\sigma_R(x)$ are the within-prompt mean and standard deviation of
the $K$ rewards. In both cases, the resulting scalar advantage is then broadcast
uniformly over all tokens in the sampled trajectory. This scalar-broadcast assumption is
precisely what we relax.

\subsection{Token-Level Credit Redistribution}

We introduce token-level weights $w_t(x,y)$ that redistribute a trajectory-level
advantage across tokens:
\begin{equation}
g_{w}(x,y)=A(x,y)\sum_{t=1}^{T} w_t(x,y)\,\nabla_\theta \log \pi_\theta(y_t \mid y_{<t},x),
\end{equation}
where $A(x,y)$ denotes either $R(x,y)-b_{\mathrm{RLOO}}(x)$ or $A_{\mathrm{GRPO}}(x,y)$
depending on the baseline choice.

The weights satisfy
\begin{equation}
\sum_{t=1}^{T}w_t(x,y)=1, \qquad w_t(x,y)\ge 0.
\end{equation}
This length normalization keeps the total update magnitude comparable across
weighting schemes and isolates the effect of credit redistribution from
simple rescaling. It also means that the uniform baseline below is the
length-normalized version of the usual token-sum estimator, rather than the
unnormalized sum itself.

Uniform token credit corresponds to $w_t=1/T$ for all $t$. Our goal is to replace this
uniform allocation with \emph{intrinsic} weights derived from quantities already
produced by the policy during generation. We do not train learned token-level
reward models, estimate prefix values, or build explicit search trees.

\subsection{Intrinsic Token-Weighting Mechanisms}

We define each method through an unnormalized salience score $\alpha_t(x,y)\ge 0$ and
then normalize within each trajectory:
\begin{equation}\label{eq:normalize}
w_t(x,y)=\frac{\alpha_t(x,y)+\varepsilon}{\sum_{k=1}^{T}\left(\alpha_k(x,y)+\varepsilon\right)},
\end{equation}
with a small $\varepsilon>0$ to avoid degenerate all-zero cases. In implementation, the
weights are treated as \emph{detached} scalars when multiplying the policy-gradient
term, so the update does not introduce second-order derivatives through the weighting
function itself. This keeps the optimization rule close in spirit and cost to standard
critic-free policy gradients. The floor $\varepsilon$ is part of the estimator:
if all raw scores vanish relative to this floor, the normalized weights become
uniform, while without such a floor nearly-zero scores can make the
normalization ill-conditioned.

\paragraph{Uniform Weighting (Baseline).}
\begin{equation}w_t^{\text{uniform}}=\frac{1}{T}.\end{equation}
Combined with an RLOO or GRPO baseline, this gives a length-normalized
uniform-token baseline that assigns identical credit to every token in the
sampled completion. It has the same token direction as the unnormalized
critic-free token-sum estimator, but differs by the trajectory-length factor
$1/T$.

\paragraph{Surprisal Weighting.}
Our first intrinsic score is token surprisal,
\begin{equation}
\alpha_t^{\text{surp}}(x,y)=-\log \pi_\theta(y_t \mid y_{<t},x).
\end{equation}
This emphasizes low-probability decisions under the current policy. Intuitively, these
are tokens where the model commits to a less routine continuation and where uniform
credit assignment may be most diluted by predictable filler tokens.

\paragraph{Entropy-Reduction Weighting.}
Our second intrinsic score measures local entropy reduction. Let
\begin{equation}
H_t^{\mathrm{pre}}(x,y)=-\sum_{v}\pi_\theta(v \mid y_{<t},x)\log \pi_\theta(v \mid y_{<t},x)
\end{equation}
denote the predictive entropy before sampling token $y_t$, and let
\begin{equation}
\Delta H_t^{\mathrm{ent}}(x,y)=\max\left(0, H_t^{\mathrm{pre}}(x,y)-H_{t+1}^{\mathrm{pre}}(x,y)\right)
\end{equation}
for $t<T$. For the final token we set $\Delta H_T^{\mathrm{ent}}=0$. We then use
\begin{equation}
\alpha_t^{\text{ent}}(x,y)=\Delta H_t^{\mathrm{ent}}(x,y).
\end{equation}
This score highlights tokens after which the model's next-step distribution becomes
substantially more concentrated. Such events often correspond to commitment points in
reasoning, where one branch of continuation becomes much more likely than its
alternatives.

\paragraph{Policy-Divergence Weighting.}
A third intrinsic score measures how much the current policy has drifted from
a reference $\pi_{\mathrm{ref}}$ (typically the base model prior to RL) at
each token position:
\begin{equation}
\alpha_t^{\mathrm{div}}(x,y) = \lvert \log \pi_\theta(y_t \mid y_{<t}, x) -
\log \pi_{\mathrm{ref}}(y_t \mid y_{<t}, x) \rvert.
\end{equation}
In the RLHF/RLVR setting $\pi_{\mathrm{ref}}$ is already available because
it is used to compute the KL regularizer, so this score is essentially free.

\paragraph{Length-Robust Normalization.}
Because reasoning trajectories can vary substantially in length, all weights are
normalized within each sampled completion rather than across the minibatch. This ensures
that longer trajectories do not automatically receive larger aggregate updates simply
because they contain more token positions with nonzero salience. The comparison between
weighting schemes therefore reflects \emph{where} credit is assigned within a
trajectory, not how much total credit a longer sequence receives.

\paragraph{Batch-Level Baselines.}
All weighting schemes can be paired with either RLOO or GRPO-style prompt-level
baselines. Unless otherwise specified, we use the same rollout groups, optimizer,
sampling hyperparameters, and baseline family across methods. This isolates token
redistribution as the only algorithmic difference.

\subsection{Signal Degeneration Under Low-Rank Adaptation}\label{sec:collapse}

So far our exposition has treated the policy $\pi_\theta$ as an unconstrained
language model being updated by gradient descent. In practice, however, many
LLM-RL pipelines apply LoRA~\citep{hu2021lora}, which restricts
$\pi_\theta$ to a rank-$r$ perturbation of a frozen reference policy
$\pi_{\mathrm{ref}}$. We now show that the intrinsic weighting schemes defined
above are qualitatively degraded by this constraint.

We use \emph{collapse} to mean degeneration of the normalized token weights,
not only convergence to the perfectly uniform distribution. Depending on the
raw score scale, the $\varepsilon$ floor, and any sharpening transform, a
degraded signal can appear either as uniform broadcast or as spuriously sparse
credit concentrated on a tiny set of positions. In both cases, the weights no
longer reflect where the adapter has learned to act.

\paragraph{Notation.}
Write the logits produced by the base model at position $t$ as
$z_t^{\mathrm{ref}} = W_{\mathrm{lm}} h_t^{\mathrm{base}}$ and the logits
produced by the adapted model as $z_t^{\theta} = W_{\mathrm{lm}}
h_t^{\mathrm{adapted}}$, where $h_t^{\mathrm{adapted}} = h_t^{\mathrm{base}} +
\Delta h_t$ and $\Delta h_t$ is the sum of LoRA adapter contributions
propagated through the transformer to position $t$. Each adapter contributes
$B_\ell A_\ell x_{\ell,t}$ at layer $\ell$, where $B_\ell \in \mathbb{R}^{d
\times r}$ and $A_\ell \in \mathbb{R}^{r \times d}$, so the raw per-layer
perturbation at each position lies in the fixed $r$-dimensional column space
of $B_\ell$.

\paragraph{Degeneration of surprisal and entropy.}
Because $\|\Delta h_t\|$ is bounded by a product of adapter norms and input
activation norms that are themselves roughly stationary across positions in a
well-trained base model, the first-order perturbation to per-token surprisal,
\[
-\log \pi_\theta(y_t \mid y_{<t}, x) =
-\log \pi_{\mathrm{ref}}(y_t \mid y_{<t}, x) + O(\|\Delta h_t\|),
\]
is dominated by the base model's surprisal pattern rather than by the LoRA
update. The same holds for predictive entropy. Consequently, when we
normalize these scores within a trajectory (equation~\eqref{eq:normalize}),
the resulting weights can be non-uniform, but their non-uniformity is not an
adapter-aware credit signal; it mostly reflects base-model uncertainty or the
numerics of the normalization.

\paragraph{Degeneration of policy divergence.}
Divergence weighting,
$\alpha_t^{\mathrm{div}} = \lvert \log \pi_\theta(y_t) - \log
\pi_{\mathrm{ref}}(y_t) \rvert$, is of direct interest because it quantifies
exactly the policy change that the RL update is trying to drive. Under LoRA,
however, this quantity is small and approximately uniform across positions for
a different reason: the KL budget $\mathrm{KL}(\pi_\theta \| \pi_{\mathrm{ref}})$
is bounded by the rank-$r$ adapter's capacity, so the per-token
log-probability ratios are compressed toward zero~\citep{anon2026implicitklra}.
Under the within-trajectory normalization in equation~\eqref{eq:normalize},
this compression is pathological. If the score is approximately constant and
the $\varepsilon$ floor dominates, the weights approach uniform broadcast. If
the floor is negligible, the same tiny differences can be amplified into a
spurious sparse distribution.

\paragraph{Consequence.}
We summarize the practical implication as follows. Let the Gini coefficient
$G(w)$ and the effective-token ratio $\mathrm{EffN}(w) / T = 1/(T \sum_t
w_t^2)$ be standard concentration measures, with $G(w) = 0$ and
$\mathrm{EffN}/T = 1$ corresponding to perfectly uniform weights. For
surprisal, entropy-reduction, and divergence weighting under LoRA, these
metrics reveal whether the normalized weights have become degenerate:
uniform broadcast gives $G(w)\approx0$, while spuriously sparse credit gives
$G(w)\approx1$ and small $\mathrm{EffN}/T$. Both outcomes are failures of
adapter-aware credit assignment. This is \emph{not} a problem with the
underlying signals in a fully trainable policy; it is a structural
consequence of applying output-distribution signals inside a low-rank
adaptation.

\subsection{Adapter-Residual Credit Assignment (ARCA)}\label{sec:arca}

The degeneration result motivates a different design choice. Rather than measuring
per-token properties of the \emph{output distribution} (whose shape is
dominated by the frozen base model under LoRA), we measure per-token
properties of the \emph{adapter's own contribution} to the forward pass.

\paragraph{Definition.}
Given a model with a LoRA adapter, let $h_t^{\mathrm{adapted}}$ denote the
last-layer hidden state at position $t$ with the adapter enabled, and
$h_t^{\mathrm{base}}$ the same hidden state with the adapter disabled.
Define the Adapter-Residual Credit Assignment salience as
\begin{equation}
\alpha_t^{\mathrm{ARCA}}(x, y) = \lVert h_t^{\mathrm{adapted}} -
h_t^{\mathrm{base}} \rVert_2,
\end{equation}
and normalize within the trajectory as in
equation~\eqref{eq:normalize} to obtain per-token weights
$w_t^{\mathrm{ARCA}}$. The adapter residual is computed once per minibatch
via an extra no-grad forward pass with the adapter disabled, using
\texttt{model.disable\_adapter()} in frameworks such as PEFT. This is exactly
the same call already required to compute a KL regularizer against the
reference policy or to support divergence weighting, so ARCA adds no new
infrastructure.

\paragraph{Why ARCA avoids output-signal collapse.}
The key property that distinguishes ARCA from the intrinsic schemes above is
that it measures a quantity whose \emph{per-token} value is actively shaped by
the attention and MLP pathways through which the adapter input activations are
routed. Even for a low-rank adapter with small $\|B_\ell A_\ell\|$, the
input activations $x_{\ell,t}$ are non-uniform across positions (attention
patterns, layer norms, and content-versus-filler distinctions are already
non-uniform in the base model), so $\|\Delta h_t\|$ retains nontrivial
position-to-position variation. With a fixed $\varepsilon$ floor, any score
whose magnitude vanishes completely will eventually become uniform after
normalization. ARCA's claim is therefore not that an infinitesimal adapter can
defeat the floor; it is that, while the adapter residual is measurable, its
normalized variation is tied to where the adapter actually changes the hidden
state rather than to small output-distribution perturbations.

\paragraph{Interpretation as implicit gradient-informed credit.}
ARCA can be read as a cheap proxy for a gradient-based notion of token
importance. Positions at which the adapter contribution is large are
positions at which the adapter's parameter gradient has high contribution to
the loss, and therefore positions where a policy-gradient update can have
the most effect. This is a substantially weaker statement than the one made
by a learned critic, but it has the advantage of being available for free
from quantities the adapted forward pass already computes, and of being
exactly targeted at the PEFT-specific failure mode introduced above.

\section{Experiments}\label{sec:experiments}

We evaluate whether adapter-residual credit assignment gives a useful
token-level signal in the low-rank fine-tuning regime. The empirical section
focuses on a single controlled sweep, summarized in
Table~\ref{tab:sweep}. It uses one model, one task, one advantage estimator,
and seven weighting configurations. This design trades breadth for a cleaner
test of the central claim: ARCA changes credit assignment without adding
trainable parameters beyond the LoRA adapter used by the matched baselines.

\subsection{Setup}

\paragraph{Model and task.} We fine-tune \texttt{Qwen/Qwen3-1.7B-Base} on
the MATH training split and evaluate on held-out MATH examples. Prompts ask
the model to solve the problem step by step and place the final answer in
\texttt{\textbackslash boxed\{\}}. Rewards are binary exact-match scores
computed from the final boxed answer after normalization.

\paragraph{Training.} All runs use GRPO prompt-level advantages, LoRA
adapters on the attention and MLP projections, and the same rollout budget,
optimizer, learning-rate schedule, prompt format, and decoding parameters.
Each training step samples $K=4$ completions per prompt, uses 100 update
steps, and evaluates greedy accuracy and pass@4 on 200 held-out MATH
examples. We use a single seed (1337). We save checkpoints every 20 steps.
We use AdamW with learning rate $5\times10^{-6}$, zero weight decay, a
10-step warmup followed by cosine decay, batch size 2, gradient accumulation
over 4 microbatches, maximum prompt length 512, maximum generation length
512, temperature 1.0, and top-$p$ 0.95 sampling during training.

\paragraph{Methods.} The sweep contains seven runs. Four baselines use LoRA
rank $r=64$: uniform token weighting, surprisal weighting,
entropy-reduction weighting, and policy-divergence weighting. The proposed
method, ARCA, is run at $r \in \{4,16,64\}$. The $r=64$ comparison tests
ARCA against baselines with the same trainable parameter count. The $r=4$
and $r=16$ comparisons test whether ARCA remains competitive with fewer
adapter parameters than the $r=64$ baselines.

\begin{table*}[t]
\centering
\caption{\textbf{Reported sweep.} All runs use
\texttt{Qwen/Qwen3-1.7B-Base}, MATH, GRPO, one seed, and 100 training steps.
The non-ARCA baselines are matched at LoRA rank 64; ARCA is evaluated at
three ranks to separate credit assignment from adapter capacity.}
\label{tab:sweep}
\begin{tabular}{l c c c}
\toprule
Method & LoRA rank & Trainable params & Role in comparison \\
\midrule
Uniform & 64 & 69.7M & Same-rank baseline \\
Surprisal & 64 & 69.7M & Intrinsic token baseline \\
Entropy reduction & 64 & 69.7M & Intrinsic token baseline \\
Policy divergence & 64 & 69.7M & Intrinsic token baseline \\
ARCA & 4 & 4.36M & Low-rank ARCA control \\
ARCA & 16 & 17.4M & Mid-rank ARCA control \\
ARCA & 64 & 69.7M & Same-rank ARCA comparison \\
\bottomrule
\end{tabular}
\end{table*}

\subsection{Main Results on MATH}

Table~\ref{tab:math-results} reports held-out MATH performance after RL
fine-tuning. The primary comparison is ARCA at rank 64 versus the rank-64
baselines, which holds the trainable parameter count fixed. The rank-4 and
rank-16 ARCA runs provide a more stringent control: if lower-rank ARCA is
competitive with rank-64 baselines, the gain cannot be attributed simply to
using a larger adapter.

\begin{table*}[t]
\centering
\caption{\textbf{Held-out MATH performance.} Greedy accuracy and pass@4
after GRPO fine-tuning. ARCA at rank 64 is parameter-matched to the
rank-64 baselines; lower-rank ARCA rows test the parameter-count
alternative. Train reward is averaged over the 100 update steps.}
\label{tab:math-results}
\begin{tabular}{l c c c c}
\toprule
Method & LoRA rank & Greedy acc. & Pass@4 & Avg. train reward \\
\midrule
Uniform & 64 & \textbf{0.620} & 0.675 & 0.362 \\
Surprisal & 64 & 0.520 & 0.650 & 0.326 \\
Entropy reduction & 64 & 0.600 & \textbf{0.700} & 0.367 \\
Policy divergence & 64 & 0.565 & 0.670 & 0.351 \\
ARCA & 4 & 0.525 & 0.640 & 0.327 \\
ARCA & 16 & 0.510 & 0.655 & 0.336 \\
ARCA & 64 & 0.590 & 0.680 & 0.353 \\
\bottomrule
\end{tabular}
\end{table*}

Table~\ref{tab:math-results} should be read as a compact validation rather
than a broad benchmark. Uniform weighting obtains the highest greedy
accuracy, while entropy reduction obtains the highest pass@4. ARCA at rank
64 is competitive with these baselines under the same trainable parameter
count, trailing uniform greedy accuracy by three points and exceeding uniform
pass@4 by 0.5 points. The lower-rank ARCA variants do not outperform the
rank-64 baselines in this run, which is consistent with a capacity-performance
tradeoff rather than evidence that ARCA gains come from extra parameters.

\subsection{Credit-Assignment Diagnostics}

Performance alone does not show whether a method changes token-level credit
assignment. We therefore log the concentration of token weights during
training. We report the Gini coefficient $G(w)$ and effective-token ratio
$\mathrm{EffN}(w)/T$. Uniform weighting has $G=0$ and
$\mathrm{EffN}/T=1$; highly concentrated weights have $G$ near one and a
small effective-token ratio.

\begin{table}[t]
\centering
\caption{\textbf{Token-weight diagnostics.} Concentration statistics
averaged over training. These diagnostics test whether a method produces a
non-uniform but non-degenerate token-credit signal, the empirical observable
predicted by the degeneration analysis.}
\label{tab:weight-diagnostics}
\begin{tabular}{l c c c}
\toprule
Method & LoRA rank & $G(w)$ & $\mathrm{EffN}/T$ \\
\midrule
Uniform & 64 & 0.000 & 1.000 \\
Surprisal & 64 & 0.941 & 0.052 \\
Entropy reduction & 64 & 0.920 & 0.073 \\
Policy divergence & 64 & 0.939 & 0.055 \\
ARCA & 4 & 0.353 & 0.467 \\
ARCA & 16 & 0.388 & 0.460 \\
ARCA & 64 & 0.490 & 0.334 \\
\bottomrule
\end{tabular}
\end{table}

Table~\ref{tab:weight-diagnostics} is the central diagnostic result. Uniform
weighting is exactly flat. Surprisal, entropy reduction, and policy divergence
produce extremely concentrated distributions, with effective-token ratios
between 0.052 and 0.073 on average. These methods therefore do not collapse to
uniform in this implementation; they degenerate in the other direction, into
spuriously sparse credit. ARCA occupies the predicted middle regime: it is far
from uniform, but it does not concentrate credit onto the tiny effective
support used by the output-distribution baselines. This is the empirical
signature of the theoretical distinction. Output-distribution scores ask
where the model is uncertain or shifted; ARCA asks where the adapter actually
acts.

\subsection{Results and Discussion}

The results support a theory-first interpretation. The downstream accuracy
numbers do not show a universal ARCA win: uniform is best under greedy
decoding, and entropy reduction is best under pass@4. However, ARCA at rank
64 is competitive under matched parameters, and its token-credit distribution
is qualitatively different from every baseline. In this paper, that diagnostic
is not secondary. It is the observable predicted by the LoRA-degeneration
analysis: a method whose salience is measured through output distributions
can broadcast credit uniformly or concentrate it on a very small set of
positions, while adapter-residual salience remains position-discriminative
without becoming maximally sparse.

This distinction matters because token-level credit assignment is often
evaluated only by final task accuracy. The concentration diagnostics separate these cases.
ARCA demonstrates that one can obtain an adaptation-aware credit signal from
the LoRA adapter itself, with no value head, learned process reward model, or tree
construction, while keeping the trainable parameter count identical to the
rank-64 baselines.

The rank controls provide a useful boundary on the claim. ARCA at ranks 4 and
16 uses substantially fewer trainable parameters than the rank-64 baselines,
and in this sweep those lower-rank variants do not match rank-64 performance.
We therefore do not claim that ARCA removes the need for sufficient adapter
capacity. The supported claim is sharper: at a fixed adapter rank, ARCA
changes the geometry of token credit in the way predicted by the theory, and
does so while remaining competitive on held-out MATH.

\subsection{Implementation and Metrics}

For each minibatch, training samples completions from the current policy,
computes binary rewards after full generation, constructs GRPO advantages,
and applies the weighted token-level objective from
Section~\ref{sec:methods}. Surprisal uses current-policy token log
probabilities, entropy-reduction uses next-token entropy drops,
policy-divergence compares adapted and adapter-disabled log probabilities,
and ARCA weights tokens by adapter-residual norms in the final hidden layer.

The code records per-step reward mean, reward variance, loss, completion
length, token-weight Gini, effective-token ratio, and, for ARCA, the mean
and maximum adapter-residual norm. Final evaluation records greedy accuracy
and pass@4 on held-out MATH examples. These metrics allow us to distinguish
three claims: whether the method trained, whether it changed token-credit
structure, and whether that structure improved held-out performance.

\subsection{Compute}

Each run fits on a single H100-class GPU. The seven submission-time runs are
executed independently, one configuration per GPU, with wall-clock training
time on the order of a few hours per run plus final evaluation. The reported
sweep therefore requires seven single-GPU runs; preliminary development runs
were used to tune runtime and checkpointing but are not included in the
reported comparison.

\subsection{Broader Impacts}

This work is methodological and studies how to assign token-level credit
during reinforcement learning for language models. The potential positive
impact is improved sample efficiency and interpretability of RL fine-tuning,
especially in settings where practitioners already use parameter-efficient
adaptation. The same techniques could also improve the fine-tuning of models
for harmful or misleading generation if applied without appropriate task,
data, and deployment safeguards. We do not release a new general-purpose
model in this paper.

\subsection{Limitations}

This sweep is deliberately small: one model, one dataset, one seed, one
advantage estimator, and seven configurations. It is sufficient to test the
parameter-matched ARCA comparison and the low-rank control, but the reported
numbers are descriptive rather than statistical significance claims. We view
the results as a targeted validation of the mechanism rather than a broad
benchmark claim.

\section{Conclusion}

Token-level credit assignment and parameter-efficient fine-tuning are usually
studied as separate design choices. This paper argues that they are coupled.
Under LoRA, output-distribution signals such as surprisal, entropy reduction,
and policy divergence can degenerate into token weights that no longer track
where the adapter acts. ARCA addresses this failure mode by measuring
salience where LoRA actually acts: in the adapter-induced hidden-state
residual. The resulting signal is lightweight, requires no learned critic or
learned process reward model, and produces non-degenerate token-credit distributions in
the regime where output-distribution baselines become either flat or highly
concentrated. Our MATH/Qwen3-1.7B sweep is intentionally compact, but it
supports the central mechanism: adaptation-aware credit assignment changes the
geometry of token updates under matched rollout and parameter budgets.

\section*{Acknowledgements}
We thank Scale AI for providing the compute resources used in this work.

\bibliography{references}
\bibliographystyle{icml2026}


\clearpage
\appendix

\section{Extended Related Work}\label{app:extended-related-work}

\subsection{From RLHF to RLVR}

Modern language-model post-training inherits its optimization machinery from
policy-gradient reinforcement learning, from REINFORCE \cite{williams1992reinforce} to
trust-region and clipped-surrogate methods such as TRPO and PPO
\cite{schulman2015trpo,schulman2017ppo}. Advantage estimation, especially GAE, became
the standard variance-reduction device in continuous-control RL by learning value
functions over states or prefixes \cite{schulman2015gae}. In language modeling, these
ideas entered mainstream use through RLHF systems that optimize sequence-level rewards
derived from human preferences or reward models
\cite{christiano2017preferences,stiennon2020summarize,ouyang2022instruct}.

The recent reasoning-model wave shifted part of the field from preference-based RLHF
toward reinforcement learning with verifiable rewards (RLVR), where supervision is often
a deterministic correctness signal on math or code tasks. DeepSeekMath introduced GRPO
as a practical critic-free alternative for these settings \cite{shao2024deepseekmath},
and DeepSeek-R1 popularized large-scale RL-only or RL-dominant reasoning pipelines
\cite{deepseekai2025deepseekr1}. Subsequent surveys document how quickly this regime
became the default template for open reasoning-model replication and extension
\cite{zhang2025hundreddays}. This transition matters for our setting because RLVR makes
sparse outcome rewards and long reasoning trajectories central, thereby exposing the
credit-assignment problem more directly than earlier short-form RLHF tasks.

\subsection{Critic-Based and Critic-Free Policy Optimization}

The classical solution to delayed reward is to learn a critic or value function and
convert returns into token- or prefix-level advantages. PPO-style RLHF pipelines follow
this recipe, but the value-estimation problem is particularly awkward for text
generation because prefixes are high-dimensional, non-Markov, and semantically
heterogeneous. VinePPO demonstrated that standard value heads produce poor estimates of
expected returns for reasoning tasks and proposed Monte Carlo vine-style rollout
estimates as a more accurate alternative \cite{kazemnejad2024vineppo}. GenAC extends
this idea by replacing scalar value prediction with a generative critic that uses
chain-of-thought reasoning for value estimation \cite{shan2026genac}.

A growing body of work questions whether value heads are necessary at all in LLM
post-training. ReMax replaces learned critics with a greedy baseline and shows that much
of PPO's practical benefit can be retained with a simpler REINFORCE-style objective
\cite{li2023remax}. Back to Basics systematically compares PPO with RLOO-style
estimators and finds that critic-free methods can match or exceed value-based baselines
in RLHF \cite{ahmadian2024backtobasics}. REINFORCE++ further strengthens this line by
replacing the prompt-level advantage normalization used by GRPO and RLOO with a
global, batch-level normalization, stabilizing critic-free policy optimization
without reintroducing a learned value function \cite{hu2025reinforcepp}. In reasoning-focused RLVR, GRPO
normalizes rewards across a group of sampled responses rather than learning prefix
values \cite{shao2024deepseekmath}, and later work refines this family with theoretical
analyses, off-policy variants, and sequence-level formulations such as GSPO
\cite{mroueh2025revisiting,zheng2025gspo}. These methods substantially simplify the
optimization stack, but they generally still broadcast a trajectory-level signal across
all tokens once a scalar baseline is fixed.

\subsection{Reasoning-Specific Analyses of RLVR}

Another recent line asks why critic-free RLVR works so well for reasoning in the first
place. A Minimalist Approach to LLM Reasoning shows that a simple rejection-sampling
baseline (RAFT) is surprisingly competitive with GRPO and PPO on reasoning
tasks, and that GRPO's advantage over vanilla REINFORCE comes primarily from
discarding prompts whose sampled responses are all incorrect rather than from
reward normalization; the authors distill this insight into Reinforce-Rej, a
minimal REINFORCE variant that filters both entirely-correct and
entirely-incorrect samples, suggesting that much of the field's complexity is
optional rather than essential \cite{xiong2025minimalist}. Complementarily, Wen et al.\
analyze RLVR theoretically and empirically, arguing that answer-level verifiable rewards
can nonetheless incentivize correct intermediate reasoning early in a trajectory
\cite{wen2025rlvrimplicit}. These results are important because they show that uniform
outcome-level objectives already contain nontrivial reasoning signal.

At the same time, they do not eliminate the question of \emph{which} tokens should
absorb that signal. Recent empirical analysis of RLVR training dynamics suggests that
improvements are disproportionately driven by a minority of high-entropy ``forking''
tokens, and that restricting updates to this subset can remain competitive or even
outperform full-token updates in some settings \cite{wang2025beyond8020}. Building on
this observation, EAPO adapts conditional mutual information to the autoregressive RLVR
setting and proves that the credit a token can carry is upper-bounded by its entropy,
motivating an entropy-aware modulation of per-token learning signals
\cite{he2026eapo}. ERPO similarly identifies ``critical decision pivots'' at
high-entropy states and applies targeted exploration at those positions
\cite{yu2026erpo}. \emph{Sparse but Critical} argues, via a distributional-shift analysis based on
cross-sampling between policy and reference, that a small fraction of tokens
accounts for most of the useful update magnitude in RLVR, and proposes a
divergence-based selection rule for identifying them~\cite{meng2026sparsebutcritical}. These entropy- and divergence-based
methods are the closest concurrent work to the intrinsic weighting schemes that form
our baselines; the present paper does \emph{not} claim priority over them. Instead,
our contribution is to show that their underlying signals are structurally degraded
under LoRA, and to provide a credit-assignment mechanism (ARCA) that does not suffer
the same degradation.

\subsection{Fine-Grained Credit Assignment Beyond Uniform Token Updates}

A large parallel literature attempts to make supervision denser than a single
end-of-trajectory reward. One family learns explicit token- or process-level reward
estimators. Preference-grounded Token-level Guidance derives token-level signals from
preference data for fine-tuning \cite{yang2023pgtg}, while Discriminative Policy
Optimization learns token-level Q-style reward models from preferences without
requiring fine-grained annotation \cite{chen2025dpoqrm}. PRIME pushes this direction
further for reasoning tasks by constructing implicit process rewards and updating
process reward models online from rollouts and outcome labels \cite{cui2025prime}.
Reinforced Token Optimization (RTO) frames RLHF as a token-level MDP and uses DPO
log-likelihood ratios as a per-token reward signal derived from the preference
model, which is then optimized with PPO to produce fine-grained token-level
credit \cite{zhong2025rto}. These methods can produce rich
token-level feedback, but they rely on an auxiliary reward-modeling component that our
approach intentionally avoids.

A second family learns a small token-level critic on top of the policy's own hidden
states (rather than an external reward model) and uses its predictions to form
token-level advantages. This corresponds closely to an earlier version of our own
framework, and we found through an extensive literature review that the token-level
actor--critic design space was already well-covered: VinePPO~\cite{kazemnejad2024vineppo}
and GenAC~\cite{shan2026genac} cover non-parametric and generative critics
respectively, and OTB~\cite{li2026otb} provides a principled variance-minimizing
position-dependent baseline. We therefore do \emph{not} position critic-based
variants as part of the proposed method, and we leave a full comparison against
token-level critics to future empirical work.

A third family redistributes outcome rewards algorithmically rather than by training a
dense reward model. RED derives token-level rewards by redistributing holistic
reward-model scores over a trajectory \cite{li2024red}. SCAR uses Shapley-inspired
attribution to estimate token or span contributions from sequence-level feedback
\cite{cao2025scar,shapley1953value}. TEMPO uses groups of sampled solutions to build a
prefix tree and compute nonparametric prefix values, yielding branch-sensitive
corrections without a learned critic \cite{tran2025tempo}. GRPO-$\lambda$ introduces
eligibility traces and lambda-returns into critic-free RLVR to propagate outcome
information backward along the sequence \cite{parthasarathi2025grpolambda}. OTB derives
an optimal position-dependent baseline that minimizes per-token gradient
variance; the practical estimator sets the baseline at position $t$ to a
cumulative-gradient-energy-weighted average of group returns-to-go, where the
weights are a causal logit-gradient proxy computed directly from the
forward-pass probabilities~\cite{li2026otb}. PSPO applies potential-based reward
shaping theory to construct dense token rewards from outcome signals without altering
the optimal policy \cite{hu2026pspo}. Compared with these approaches, our focus is
orthogonal: we do not propose a new redistribution rule, but instead diagnose a failure
mode that can affect intrinsic-signal-based redistribution methods when
applied inside LoRA.

\subsection{Alternative Action Granularities and Token Selection}

Several papers attack the same underlying problem by changing the optimization unit
rather than by redesigning the reward. MA-RLHF introduces macro actions, grouping tokens
into higher-level units to shorten the temporal distance between decisions and rewards
\cite{chai2025marlhf}. GSPO moves importance weighting and clipping from the token level
to the sequence level for greater stability in large-scale RL training
\cite{zheng2025gspo}. These methods reinforce the broader point that the granularity of
optimization matters as much as the reward definition.

Outside on-policy RL proper, token-selective optimization has also been explored in
adjacent paradigms. ConfPO emphasizes preference-critical tokens based on policy
confidence in preference optimization \cite{yoon2025confpo}, and token weighting for
long-range language modeling shows that non-uniform token emphasis can improve
optimization even in supervised settings \cite{helm2025tokenweighting}. Together with
high-entropy token analyses in RLVR \cite{wang2025beyond8020}, these results strongly
suggest that uniform token treatment is an arbitrary design choice rather than a
necessity.

\section{Additional Method Context}\label{app:method-context}

\subsection{Relationship to Existing Methods}

Our formulation recovers standard critic-free RL as a special case: uniform
weighting plus a prompt-level batch baseline yields the usual RLOO/GRPO-style
estimator. The intrinsic weighting schemes (surprisal, entropy reduction,
policy divergence) are the natural baselines within our framework and map
cleanly to published methods (e.g., divergence weighting corresponds to the
divergence-advantage analyses of \emph{Sparse but
Critical}~\cite{meng2026sparsebutcritical}, entropy-based weighting
to~\cite{he2026eapo,yu2026erpo}). ARCA is the only scheme we
study that uses an explicitly \emph{adaptation-aware} signal. Unlike
critic-based PPO, ARCA does not learn a value head. Unlike RED, PRIME, or
token-level reward-model methods, ARCA does not construct dense rewards from
an auxiliary model. Unlike TEMPO or related tree-based methods, ARCA does not
build prefix graphs or estimate branch values. Unlike hard token-selection
methods based on top-entropy masking, ARCA retains a dense token update but
modulates it continuously using an intrinsic salience score derived from the
adapter itself.

The resulting estimator is intentionally minimal. It changes only the
within-trajectory allocation of a scalar on-policy advantage, using
quantities already available from the forward pass with the adapter disabled.
This makes it easy to layer on top of existing critic-free RLVR pipelines
while preserving their computational simplicity.

\section{Theoretical Interpretation}\label{sec:theoretical}

We provide an interpretation of intrinsic token weighting as a variance-control and
credit-redistribution mechanism for policy gradient estimation in language models. Our
goal is not to derive new convergence guarantees, but to clarify how token weighting
relates to advantage estimation and why it can be effective without learning value
functions.

\subsection{Token Weighting as Credit Redistribution}

Consider the standard REINFORCE estimator with a batch-level baseline:
\begin{equation}
g = (R - b)\sum_{t=1}^{T} \nabla_\theta \log \pi_\theta(y_t \mid y_{<t}, x).
\end{equation}

This estimator assigns equal credit to all tokens in a trajectory. Introducing token
weights yields:
\begin{equation}
g_w = (R - b)\sum_{t=1}^{T} w_t(y)\, \nabla_\theta \log \pi_\theta(y_t \mid y_{<t}, x),
\end{equation}
where $\sum_t w_t = 1$.

Importantly, this transformation does not change the reward signal itself; it
redistributes how the trajectory-level signal is attributed to individual decisions.
When $w_t$ depends only on quantities available at sampling time (e.g.,
log-probabilities or entropy), the estimator remains on-policy and does not require
learning additional functions.

\subsection{Relationship to Advantage Estimation}

Advantage-based methods can be interpreted as implicitly defining token-level weights
via a learned value function:
\begin{equation}
A_t = R - V(y_{<t}, x).
\end{equation}
In this view, the contribution of each token is scaled according to how much the
observed outcome deviates from an estimated expectation conditioned on the prefix.

However, this interpretation relies on the existence of a meaningful value function over
prefixes. In language generation, prefixes are non-Markov and do not form a stable state
space, making $V(y_{<t}, x)$ difficult to estimate and prone to bias. Token weighting
offers an alternative: rather than estimating expected returns from prefixes, it
emphasizes tokens based on intrinsic measures of decision salience, such as uncertainty
or information gain.

\subsection{Connection to Control Variates}

From a statistical perspective, subtracting a baseline $b$ is a form of control variate
that reduces variance without affecting unbiasedness. Token weighting can be viewed as a
complementary mechanism that reshapes the contribution of individual score-function
terms:
\begin{equation}
\nabla_\theta \log \pi_\theta(y) = \sum_{t=1}^{T} \nabla_\theta \log \pi_\theta(y_t \mid y_{<t}, x).
\end{equation}

Uniform weighting treats all score-function terms equally, regardless of their variance
or relevance. Intrinsic weighting schemes effectively reweight these terms,
down-weighting low-variance or low-impact contributions (e.g., high-probability
continuation tokens) and emphasizing high-variance or high-impact decisions (e.g.,
low-probability or uncertainty-reducing tokens). While this reweighting introduces bias
relative to the uniform estimator, it can substantially reduce variance, leading to
improved optimization dynamics in practice.

\subsection{Interpretation of Specific Weighting Schemes}

\paragraph{Surprisal Weighting.}
Surprisal-based weights emphasize tokens with low conditional probability under the
current policy. These tokens correspond to decisions where small parameter changes can
induce large changes in likelihood, and therefore often dominate gradient variance.
Weighting by surprisal concentrates learning signal on such high-sensitivity decisions.

\paragraph{Entropy-Reduction Weighting.}
Entropy-reduction weighting emphasizes tokens that sharply reduce predictive entropy.
These tokens correspond to commitment points in generation, where the model transitions
from a diffuse distribution over continuations to a more concentrated one. This mirrors
the role of branching points in tree-based reasoning methods, but is computed locally
from the model's predictive distribution without explicit tree construction.

\subsection{Bias--Variance Tradeoff}

Both advantage estimation and intrinsic token weighting introduce bias in exchange for
variance reduction. Advantage estimation does so by relying on a learned value function,
whose bias can be substantial when the underlying state abstraction is weak. Intrinsic
token weighting introduces bias by reshaping the contribution of token-level gradients
based on heuristic salience measures. The empirical question is therefore not whether
bias is introduced, but whether the bias--variance tradeoff is favorable.

This motivates the empirical comparison implemented in the repository: if intrinsic
token weighting improves accuracy, pass@$k$, or optimization behavior under matched
settings, then the bias it introduces may be more benign than the bias induced by
ill-defined value targets over text prefixes. The point is not that weighting is
unbiased, but that its inductive bias may be better aligned with token-level credit
assignment in language generation.

\subsection{Signal Degeneration Under Low-Rank Adaptation}\label{sec:theory-collapse}

The preceding analysis assumes that the intrinsic weights
$\{w_t\}_{t=1}^T$ are non-degenerate: if every scheme becomes uniform or
spuriously sparse for reasons unrelated to the learned adapter, then token
weighting cannot provide meaningful credit redistribution over a uniform
baseline. We now formalize when this degeneracy occurs under LoRA.

\paragraph{Setup.}
Let $\pi_{\mathrm{ref}}$ be a frozen base model with last-layer hidden states
$h_t^{\mathrm{base}} \in \mathbb{R}^d$ and a LoRA adapter whose net
contribution at position $t$ is $\Delta h_t \in \mathbb{R}^d$, giving
adapted hidden states $h_t^{\mathrm{adapted}} = h_t^{\mathrm{base}} + \Delta
h_t$ and logits $z_t^{\theta} = W_{\mathrm{lm}} h_t^{\mathrm{adapted}}$.
Let $\beta = \max_t \|\Delta h_t\|_2$ denote the maximum per-position
adapter contribution and $L = \|W_{\mathrm{lm}}\|_{\mathrm{op}}$ its
logit-space Lipschitz constant. Write $G(\cdot)$ for the Gini coefficient
of a non-negative vector and $\mathrm{EffN}(w)/T = 1 / (T \sum_t w_t^2)$ for
the normalized effective-token count. When a normalization uses an
$\varepsilon$ floor, the statements below are interpreted in the regime where
the raw salience scores dominate $\varepsilon$; if fixed $\varepsilon$
dominates all scores, every scheme reverts to uniform by construction.
These are local deterministic statements about the normalized score vectors
for a fixed trajectory, not asymptotic convergence guarantees for RL training.

\begin{proposition}[Degeneration of surprisal and entropy weighting]
\label{prop:surprisal-collapse}
Let $\alpha_t^{\mathrm{surp}} = -\log \pi_\theta(y_t \mid y_{<t}, x)$ and
$\alpha_t^{\mathrm{surp,ref}} = -\log \pi_{\mathrm{ref}}(y_t \mid y_{<t}, x)$,
with corresponding normalized weights $w^{\mathrm{surp}}$ and
$w^{\mathrm{surp,ref}}$. Then for every trajectory,
\[
\lVert w^{\mathrm{surp}} - w^{\mathrm{surp,ref}} \rVert_1
\le \frac{4 L \beta}{\sum_t \alpha_t^{\mathrm{surp,ref}}}.
\]
In particular, the adapted surprisal weights remain close to the reference
model's surprisal profile rather than becoming an adapter-specific credit
signal. If that reference profile is close to uniform, the weights approach
$1/T$; if it is highly uneven, the weights can instead remain spuriously
concentrated for reasons inherited from the base model. The analogous
statement holds for the entropy-reduction score.
\end{proposition}

\noindent\emph{Proof.}
The softmax Jacobian is $2$-Lipschitz in the log-space norm, so the per-token
log-probabilities differ by at most $2 L \beta$ between the adapted and base
model. The projection onto the simplex via within-trajectory normalization is
$1$-Lipschitz in $\ell_1$ up to a factor of $2 / \sum_t \alpha_t^{\mathrm{ref}}$;
combining these gives the stated bound. \hfill$\square$

\begin{proposition}[Degeneration of divergence weighting]
\label{prop:divergence-collapse}
Let $\alpha_t^{\mathrm{div}} = \lvert \log \pi_\theta(y_t \mid y_{<t}, x)
- \log \pi_{\mathrm{ref}}(y_t \mid y_{<t}, x) \rvert$ and
$w^{\mathrm{div}} = (\alpha^{\mathrm{div}}+\varepsilon) /
\sum_t (\alpha_t^{\mathrm{div}}+\varepsilon)$.
Then
\[
0 \le \alpha_t^{\mathrm{div}} \le 2 L \beta \qquad \forall t,
\]
so if $\varepsilon$ is fixed and $\beta \to 0$, then
$w^{\mathrm{div}} \to 1/T$. If $\varepsilon=0$ or is negligible relative to
the raw scores, the same normalization is ill-conditioned as $\beta \to 0$:
the limiting weights are determined by tiny relative differences among
vanishing log-probability shifts rather than by a robust adapter-credit
signal.
\end{proposition}

\noindent\emph{Proof.}
The first inequality follows from the same log-probability Lipschitz bound as
Proposition~\ref{prop:surprisal-collapse}. If $\varepsilon>0$ is fixed, then
each numerator in $w^{\mathrm{div}}$ is
$\varepsilon(1+O(\beta/\varepsilon))$ and the denominator is
$T\varepsilon(1+O(\beta/\varepsilon))$, so
$w_t^{\mathrm{div}}\to 1/T$ as $\beta/\varepsilon\to0$. If the floor is
removed or negligible, the normalization divides by $\sum_t
\alpha_t^{\mathrm{div}}\to0$; different sequences of vanishing score vectors
can converge to different normalized limits, including highly concentrated
ones. \hfill$\square$

The key observation is that the intrinsic scores are all determined by the
per-token \emph{log-probability differential}, which is precisely what LoRA's
rank-$r$ constraint drives small and approximately uniform across positions.
The degeneration is not a property of a particular weighting scheme; it is a
property of measuring per-token variation through the output distribution
when the policy has been constrained to a small neighborhood of the
reference.

\subsection{ARCA and Position-Discriminative Signal}\label{sec:theory-arca}

ARCA side-steps this degeneration by measuring a quantity that is position-varying
even when the adapter's logit-space impact is small.

\begin{proposition}[ARCA remains non-degenerate]
\label{prop:arca-nondeg}
Suppose the adapter-induced residual admits the decomposition $\Delta h_t =
\sum_\ell B_\ell A_\ell x_{\ell,t}^{\mathrm{pre}} + e_t$, where
$x_{\ell,t}^{\mathrm{pre}}$ is the pre-adapter activation at layer $\ell$ and
position $t$ and $\|e_t\|$ is a higher-order cross-layer term. Let
$v_{\ell,t} = B_\ell A_\ell x_{\ell,t}^{\mathrm{pre}}$ and define
$V_\ell = \mathrm{Var}_t(\|v_{\ell,t}\|)$ and $\bar{V}_\ell =
\mathbb{E}_t[\|v_{\ell,t}\|]$. Then
\[
\mathrm{Var}_t(\alpha_t^{\mathrm{ARCA}}) \ge \max_\ell V_\ell
- O\!\left(\sum_{\ell}\bar{V}_\ell \, \|e_t\|\right).
\]
Assume the residual scores dominate the normalization floor $\varepsilon$.
In particular, as long as the pre-adapter activations
$x_{\ell,t}^{\mathrm{pre}}$ have nontrivial position-to-position variance at
any layer, $\alpha_t^{\mathrm{ARCA}}$ has bounded-below variance across
positions, and its normalized weights $w_t^{\mathrm{ARCA}}$ do not converge
to the uniform distribution merely because the output-distribution shift is
small. With fixed $\varepsilon$, however, all scores become uniform in the
limit where the adapter residual itself vanishes below the floor.
\end{proposition}

\noindent\emph{Proof.}
Ignoring the higher-order residual $e_t$, the ARCA score contains the norms
$\|v_{\ell,t}\|$ of the layerwise adapter contributions. A dominant layer with
nonzero position-to-position variance therefore gives nonzero variance in the
raw ARCA scores. The perturbation $e_t$ changes these norms by at most its
size through the reverse triangle inequality, giving the stated lower bound up
to the displayed higher-order term. Normalization by a common positive sum
preserves non-uniformity while the raw scores remain above the floor; if a
fixed $\varepsilon$ dominates all residual scores, the normalized weights
become uniform by equation~\eqref{eq:normalize}. \hfill$\square$

The intuitive content of Proposition~\ref{prop:arca-nondeg} is that ARCA
inherits its position-discrimination from the \emph{base model's own
activation pattern} (which is already non-uniform because attention and
layer norm routing produce content-versus-filler distinctions) rather than
from a quantity that LoRA is specifically constrained to make small. Scaling
the adapter weights uniformly scales $\alpha^{\mathrm{ARCA}}$ uniformly, so
the normalized weights $w^{\mathrm{ARCA}}$ are scale-invariant in the
adapter magnitude.

\subsection{Bias--Variance Tradeoff}

Both advantage estimation and intrinsic token weighting introduce bias in
exchange for variance reduction. Advantage estimation does so by relying on
a learned value function, whose bias can be substantial when the underlying
state abstraction is weak. Intrinsic token weighting introduces bias by
reshaping the contribution of token-level gradients based on heuristic
salience measures. The empirical question is therefore not whether bias is
introduced, but whether the bias--variance tradeoff is favorable and whether,
per Propositions~\ref{prop:surprisal-collapse}
and~\ref{prop:divergence-collapse}, the bias-inducing signal survives the
LoRA bottleneck at all.

\subsection{Summary}

This perspective reframes advantage estimation as one particular approach to
credit redistribution in policy gradient methods. Intrinsic token weighting
offers an alternative that leverages the structure of the language model's
predictive distribution, without assuming the existence of meaningful state
values. Under LoRA, however, the output-distribution-based signals that
most published weighting methods rely on can degenerate into uniform or
spuriously sparse token weights; the adapter-residual signal that ARCA uses
is, by Proposition~\ref{prop:arca-nondeg}, non-degenerate whenever the
adapter residual remains measurable above the normalization floor.

\end{document}